\documentclass[12pt]{article}
\usepackage{epsfig}
\usepackage{amsmath,amsthm} 
\usepackage{amssymb}
  \usepackage[top=3cm, bottom=3cm, left=2cm, right=2cm]{geometry}
  
\makeatletter
\def\@makechapterhead#1{
\vspace*{80\p@}
{\parindent \z@ \raggedright \normalfont
\ifnum \c@secnumdepth >\m@ne
\if@mainmatter
\huge\bfseries \@chapapp\space \thechapter
\par\nobreak
\vskip 20\p@
\fi
\fi
\interlinepenalty\@M
\Huge \bfseries #1\par\nobreak
\vskip 40\p@
}}
\makeatother

\DeclareMathOperator{\R}{\mathbb{R}}

\DeclareMathOperator{\E}{\mathbb{E}}

\renewcommand{\P}{\mathbb{P}}

\DeclareMathOperator{\F}{F}

\newtheorem{theorem}{Theorem}

\newtheorem{lemma}{Lemma}

\newtheorem{remark}{Remark}
\newtheorem{ex}{Example}

\newcommand{\new}{\newcommand}

\providecommand{\nor}[1]{\left\lVert {#1} \right\rVert}

\providecommand{\scal}[2]{\left\langle{#1},{#2}\right\rangle}

\new{\N}{\mathbb N}
\new{\Z}{\mathbb Z}

\newcommand{\hh}{\mathcal H}

  \newtheorem{thm}{Theorem}[section]

  \newtheorem{lem}{Lemma}[section]
  \newtheorem{defn}{Definition}[section]

  \newcommand{\cv}{{\mbox{CV}_{loo}}}

  \newcommand{\cvon}{{\mbox{CV}_{on}}}

   \title{Online Learning, Stability, and Stochastic Gradient Descent}

\begin{document}

\maketitle
   
  \bibliographystyle{plain} 
  \thispagestyle{empty} 
  \vspace{0.1cm} 
 \begin{center} 
    {\bf  Tomaso Poggio, Stephen Voinea,
Lorenzo Rosasco}
 \end{center} 
  \begin{center} 
   {\it CBCL, McGovern Institute, CSAIL, Brain Sciences Department, Massachusetts Institute of Technology


}
  \end{center} 

  \vspace{0.2in} 
\pagenumbering{arabic}
  \begin{abstract}   
    
    {\noindent In batch learning, stability together with existence
      and uniqueness of the solution corresponds to well-posedness of
      Empirical Risk Minimization (ERM) methods; recently, it was
      proved that {\it $\cv$ stability} is necessary and sufficient
      for generalization and consistency of ERM (\cite{PRMN04}). In this
      note, we introduce {\it $\cvon$ stability}, which plays a similar role in 
     online learning. We show that stochastic gradient descent (SDG) with 
      the usual hypotheses is $\cvon$ stable and we then discuss the 
      implications of $\cvon$ stability for convergence of SGD.}

\end{abstract} 
  \vspace{0.1cm} 
{\footnotesize \noindent
     This report describes research done within the Center for Biological \&
     Computational Learning in the Department of Brain \& Cognitive Sciences
     and in the Artificial Intelligence Laboratory at the Massachusetts
     Institute of Technology.     
     This research was sponsored by grants from:   AFSOR, DARPA, NSF.      
     Additional support was provided by: Honda R\&D Co., Ltd., Siemens
     Corporate Research, Inc., IIT, McDermott Chair.
  \vfill \par} 
\newpage  

\tableofcontents
\vspace{0.5cm}

%
%
%
%
%
%
%
%

\vspace{0.5cm}

\section{Learning, Generalization and Stability}

In this section we collect some basic definition and facts.  

\subsection{Basic Setting}

Let $Z$ be a probability space with a measure $\rho$.  A training  set $S_n$ is an i.i.d. sample $z_i$, $i, =0,\dots, n-1$ from $\rho$.  
Assume that a hypotheses space ${\cal H}$ is given. We typically assume $\hh$ to be a Hilbert space and sometimes   a $p$-dimensional Hilbert Space, in which case, without loss of generality, we  identify elements in $\hh$ 
with $p$-dimensional vectors and ${\cal H}$ with $\R^p$.
A loss function is a map $V:\hh\times Z\to \R_+$. 
Moreover we assume that 
$$ 
I(f)=\E_z V(f,z),
$$
exists and is finite for $f\in \cal H$. We consider the problem of finding a minimum of $I(f)$ in $\hh$. 
In particular, we restrict ourselves to finding a minimizer of $I(f)$ in a closed subset $K$ of ${\cal H}$ (note that we can of course have $K = {\cal H}$). We denote this minimizer by $f_K$ so that 
$$
I(f_K)=\min_{f\in K}I(f).
$$
Note that in general, existence (and uniqueness) of a minimizer is not guaranteed unless some further assumptions are specified.

\begin{ex}
An example of the above set is supervised learning. In this case  $X$ is usually 
a subset of $\R^d$ and $Y=[0,1]$. There is a Borel probability measure $\rho$ on $Z=X\times Y$.
and $S_n$ is an i.i.d. sample $z_i=(x_i,y_i)$, $i, =0,\dots, n-1$ from $\rho$.  
The  hypotheses space ${\cal H}$ is  a space of functions from $X$ to $Y$ and a typical example 
of loss functions is the square loss $(y-f(x))^2$.
\end{ex}

\subsection{Batch and Online Learning Algorithms}
A batch learning algorithm $A$  maps a training set to a function in the hypotheses space, that is $$f_n=A(S_n)\in \cal H,$$
and is typically assumed to be {\em symmetric}, that is, invariant to permutations in the training set.
An online learning algorithm is defined recursively as $f_0=0$ and $$f_{n+1}=A(f_n, z_n).$$ 
A weaker notion of an online algorithm is $f_0=0$ and $f_{n+1}=A(f_n, S_{n+1}).$ The former definition gives  a 
memory-less algorithm, while the latter keeps memory of the past (see \cite{Rakh10}). 
Clearly, the algorithm obtained from either of these two procedures will not in general be symmetric.

\begin{ex}[ERM]
The prototype example of batch learning algorithm is empirical risk minimization, defined by the variational problem
$$
\min_{f\in \hh} I_n(f),
$$ 
where $I_n(f)=\E_n V(f,z)$,  $\E_n$ being the empirical average on the sample, and  $\hh$ is typically assumed to be a proper, closed subspace of $\R^p$, for example  a ball or the convex hull of some given finite set of vectors.
\end{ex}

\begin{ex}[SGD]
The prototype example of online  learning algorithm is stochastic gradient descent, defined by the recursion 
\begin{equation}\label{sgdK}
f_{n+1}=\Pi_K(f_n - \gamma_ n \nabla V(f_n, z_n)),
\end{equation}
where $z_n$ is fixed, $\nabla V(f_n, z_n)$ is the gradient of the loss with 
respect to $f$ at  $f_n$, and $\gamma_n$ is a suitable decreasing sequence. Here $K$ is 
 assumed to be a closed subset of $\hh$ and $\Pi_K:\hh\to K $ the corresponding  projection. 
 Note that if $K$ is convex then $\Pi_K$ is a contraction, i.e. $\nor{\Pi_K}\le 1$ and moreover if $K=\hh$ then $\Pi_K=I$.
\end{ex}

\subsection{Generalization and Consistency}
In this section we discuss several ways of formalizing the concept of generalization of a learning algorithm.
We say that an algorithm is weakly consistent if we have convergence of the risks in probability, that is  for all $\epsilon>0$, 
\begin{equation}\label{wcons}
\lim_{n\to \infty} \P(I(f_n)-I(f_K)>\epsilon)=0,
\end{equation}
and that it is strongly consistent if convergence holds almost surely, that is
$$
 \P\left(\lim_{n\to \infty} I(f_n)-I(f_K)=0\right)=1.
$$
A different notion of consistency, typically considered in statistics,  is given by convergence  
in expectation 
$$
\lim_{n\to \infty} \E[I(f_n)-I(f_K)]=0.
$$
Note that, in the above equations, probability and expectations 
are with respect to the sample $S_n$.
We add three  remarks.
\begin{remark}
A more general requirement than those described above is obtained by replacing $I(f_K)$ by 
$\inf_{f\in \hh} I(f)$. Note that in this latter case no extra assumptions are needed.
\end{remark}
\begin{remark}
Yet a more general requirement would be obtained by replacing $I(f_K)$ by 
$\inf_{f\in \cal F} I(f)$, $\cal F$ being the largest space such that $I(f)$ is defined. 
An algorithm having such a consistency property is called universal. 
\end{remark}
\begin{remark}
We note that, following \cite{Alonetal97}  the convergence \eqref{wcons} corresponds to the definition of learnability of the class $\hh$. 
\end{remark}

\subsubsection{Other Measures of Generalization.} 
Note that alternatively one could measure the error with respect to the norm in $\hh$, that is $\nor{f_n-f_K}$, for example
\begin{equation}\label{norm}
\lim_{n\to \infty} \P(\nor{f_n-f_K}>\epsilon)=0.
\end{equation}
A different requirement is to have convergence  in the form 
\begin{equation}\label{empexp}
\lim_{n\to \infty} \P(|I_n(f_n)-I(f_n)|>\epsilon)=0.
\end{equation}
Note that for both the above error measures one can consider different notions of convergence (almost surely, in expectation)
as well convergence rates, hence finite sample bounds. 

For certain algorithms, most notably ERM, under mild   assumptions on the loss functions, the 
 convergence \eqref{empexp} implies weak consistency\footnote{
In fact for ERM
\begin{eqnarray*}
\P(I(f_n)-I(f_K)>\epsilon) &=&\P( I(f_n)-I_n(f_n)+I_n(f_n)-I_n(f_K)+I_n(f_K)-I(f_K)>\epsilon)\\
&\le&\P( I(f_n)-I_n(f_n)>\epsilon) + \P(I_n(f_n)-I_n(f_K)>\epsilon)+\P(I_n(f_K)-I(f_K)>\epsilon)
\end{eqnarray*}
The first term goes to zero because of  \eqref{empexp}, the second term has probability zero since $f_n$ minimizes  $I_n$, 
the third term goes to zero if $V(f_K,z)$ is a well behaved random variable (for example if the loss is bounded but also 
under weaker moment/tails conditions).}. 
For general algorithms there is no straightforward connection between \eqref{empexp} and consistency \eqref{wcons}. 

Convergence \eqref{norm} is typically stronger than \eqref{wcons}, in particular this can be seen if 
the loss  satisfies the Lipschitz condition
\begin{equation}\label{lip}
|V(f,z)-V(f',z)|\le L\nor{f-f'}, \quad \quad L>0,
\end{equation}
for all $f,f'\in \hh$ and $z\in Z$, but also for other loss function which do 
not satisfy  \eqref{lip} such as the square loss.

\subsection{Stability and Generalization}

Different notions of stability are sufficient to imply consistency results as well as finite sample bounds. 

A strong form of stability is uniform stability
$$
\sup_{z\in Z}\sup_{z_1,\dots,z_n}\sup_{z'\in Z} |V(f_n, z)-V(f_{n,z'},z)| \le \beta_n
$$
where $f_{n,z'}$ is the function returned by an algorithm if we replace the $i$-th point in $S_n$ by $z'$ and $\beta_n$ is a 
decreasing function of $n$.

Bousquet and Eliseef prove that the above condition, for algorithms
which are {\em symmetric}, gives exponential tail inequalities on
$I(f_n)-I_n(f_n)$ meaning that we have $\delta(\epsilon,
n)=e^{-C\epsilon^2n}$ for some constant $C$ \cite{BE:2001}.  Furthermore, it was shown in \cite{WibRosPog} that ERM with a strongly convex loss function is always uniformly
stable.  Weaker requirements can be defined by replacing one or more
supremums with expectation or statements in probability; exponential
inequalities will in general be replaced by weaker concentration. A
thorough discussion and a list of relevant references can be found in
\cite{RMP05, MNPR06}. Notice that the notion of {\it $\cv$ stability}
introduced there is necessary and sufficient for generalization and
consistency of ERM (\cite{PRMN04}) in the batch setting of
classification and regression. This is the main motivation for
introducing the very similar notion of $\cvon$ stability for the
online setting in the next section\footnote{Thus for the setting of
  batch classification and regression it is not necessary
  (S. Shalev-Schwartz, pers. comm.) to use the framework of
  \cite{SSSOSNSKS10}).}-

\section{Stability and SGD}
Here we focus on online learning and in particular on SGD and discuss the 
role played by the following definition of stability, that we call  $\cvon$ stability
\begin{defn}
We say that an online algorithm is $\cvon$ stable with rate $\beta_n$ if for $n>N$ we have 
\begin{equation}\label{CVON1_tp}
~ - \beta_n \leq {\E_{z_n}}[  V(f_{n+1}, z_{n})-V(f_{n}, z_{n})|S_n]
< 0,
\end{equation}
where $S_n= z_0, \dots, z_{n-1}$ and  $\beta_n \geq 0$ goes to zero with $n$.
\end{defn}

The definition above is of course equivalent to 

\begin{equation}\label{CVON1}
0 < {\E_{z_n}}[  V(f_{n}, z_{n})-V(f_{n+1}, z_{n})|S_n] \le ~ \beta_n.
\end{equation}

In particular,  we  assume $\hh$ to be    a $p$-dimensional Hilbert Space and  $V(\cdot, z)$ 
to be convex and twice differentiable in the first argument for all values of $z$.
We discuss the stability property of \eqref{sgdK} when $K$ is a closed, convex subset; in particular, we focus on the case when we can drop the projection so that 

\begin{equation}\label{sgd}
f_{n+1}=f_n - \gamma_ n \nabla V(f_n, z_n).
\end{equation}

\subsection{Setting and Preliminary Facts}

We recall the following standard result, see \cite{Lelong05} and references therein for a proof. 
\begin{theorem}\label{sgdconv}
Assume that, 
\begin{itemize}
\item  There exists $f_K\in K$, such that  $\nabla I(f_K)=0$,
and  for all $f\in \hh$, $\scal{f-f_K}{\nabla I(f)}>0$. 
\item  $\sum_{n}\gamma_n=\infty$, $\sum_{n}\gamma^2_n<\infty$.
\item There exists $D>0$, such that for all $f_n\in \hh$, 
\begin{equation}\label{local}
\E_{z_n} [\nor{\nabla V(f_n,z)}^2|S_n]\le D(1+\nor{f_n-f_K}^2).
\end{equation}
\end{itemize}
Then,  
$$
\P(\lim_{n\to \infty}\nor{f_n-f_K}=0)=1.
$$
\end{theorem}
The following result  will be also useful.
\begin{lemma}\label{proj}
Under the same assumptions of Theorem \ref{sgdconv}, if $f_K$ belongs to the interior of $K$, then there exists $N>0$ such that
for $n>N$, $f_n\in K$ so that the projections of \eqref{sgdK} are
not needed and the $f_n$ are given by $f_{n+1}=f_n - \gamma_ n \nabla V(f_n, z_n)$.
\end{lemma}

\subsubsection{Stability of SGD}

Throughout this section we assume that 
\begin{equation}\label{hess}
\scal{f}{H(V(f,z))f}\ge 0\quad  \text\quad \nor{H(V(f,z))}\le M<\infty,
\end{equation}
for any $f\in \hh$ and $z\in Z$; $H(V(f, z))$ is the Hessian of $V$.

\begin{theorem}
Under the same assumption of Theorem \ref{sgdconv}, there exists $N$ such that for $n>N$, 
SGD satisfies $\cvon$ with $\beta_n=C\gamma_n$, where $C$ is a universal constant.
\end{theorem}
\begin{proof}
Note that from Taylor's formula, 
\begin{eqnarray} \label{Duflo} 
 [V(f_{n+1},z_{n}) - V(f_{n},z_{n})] =
 \scal{f_{n+1} - f_n}{ \nabla V(f_{n}, z_n)} + 1/2  \scal{f_{n+1} - f_n}{H(V(f, z_n)) (f_{n+1} - f_n)},
 \end{eqnarray}
 with  $f=\alpha f_n+ (1-\alpha)f_{n+1}$ for $0\le \alpha\le 1$.
 We can use the definition of SGD and Lemma \ref{proj} to show there exists $N$ such that for $n>N$,   
$f_{n+1} - f_n=\gamma_n V(f_{n},z_n)$.
Hence  changing signs in \eqref{Duflo} and taking the expectation  w.r.t. $z_n$ conditioned over $S_n=z_0, \dots, z_{n-1}$,
we get
\begin{eqnarray}  \label{tp1}
 &~&{\E_{z_n}}[V(f_{n}, z_{n})  -
  V(f_{n+1}, z_{n})| S_n] = \notag\\\notag\\
&~& \gamma_n \E_{z_n}[
\nor{\nabla V(f_{n},z_n)}^2 |S_n] + 1/2  \gamma_n^2 \E_{z_n} [\scal{\nabla V(f_{n},z_n)}
{H(V(f, z_n))\nabla_f V(f_{n},z_n)} |S_n].
\end{eqnarray} 
The above quantity is clearly non negative, in particular the last term is non negative because of \eqref{hess}.
Using \eqref{local} and \eqref{hess} we get
$$
 {\E_{z_n}}[V(f_{n}, z_{n})  -
  V(f_{n+1}, z_{n})| S_n] 
 = (\gamma_n +1/2\gamma_n^2 M) D(1+  \E_{z_n} [ \nor{f_n-f_K}  |S_n])\le C\gamma_n,  
$$
if  $n$ is large enough.
\end{proof}

A partial converse result is given by the following theorem. 

\begin{theorem}
Assume that, 
\begin{itemize}
\item  There exists $f_K\in K$, such that  $\nabla I(f_K)=0$,
and  for all $f\in \hh$, $\scal{f-f_K}{\nabla I(f)}>0$. 
\item  $\sum_{n}\gamma_n=\infty$, $\sum_{n}\gamma^2_n<\infty$.
\item There exists $C, N>0$, such that for all $n>N$, \eqref{CVON1}
  holds with $\beta_n = C\gamma_n$.
\end{itemize}
Then,  
\begin{equation}\label{convsgd}
\P\left(\lim_{n\to \infty}\nor{f_n-f_K}=0\right)=1.
\end{equation}

\end{theorem}
\begin{proof}

Note that from \eqref{Duflo}  we also have 
\begin{eqnarray*}
&~&{\E_{z_n}}[V(f_{n+1}, z_{n})  -
  V(f_{n}, z_{n})| S_n] =\\\\
 &~&  -\gamma_n \E_{z_n}[
\nor{\nabla V(f_{n},z_n)}^2 |S_n] + 1/2  \gamma_n^2 \E_{z_n} [\scal{\nabla V(f_{n},z_n)}
{H(V(f, z_n))\nabla_f V(f_{n},z_n)} |S_n].
\end{eqnarray*}
so that   using the  stability assumption and \eqref{hess} we obtain,  
$$
- \beta_n  \leq (1/2\gamma_n^2 - \gamma_n)\E_{z_n}[\nor{\nabla V(f_{n},z_n)}^2 |S_n]
$$
that is, 
$$
\E_{z_n}[\nor{\nabla V(f_{n},z_n)}^2 |S_n]\le
\frac{\beta_n}{(\gamma_n-M/2\gamma_n^2)} = \frac{C\gamma_n}{(\gamma_n-M/2\gamma_n^2)}. 
$$

From Lemma \ref{proj} for $n$ large enough we obtain
\begin{eqnarray*}
\nor{f_{n+1} - f_K}^2  &\le&  \nor{ f_{n}- \gamma_n\nabla V(f_n,z_n)- f_K}^2\\
 &\leq& \nor{f_{n} - f_K}^2  + \gamma_n^2 \nor{\nabla V(f_n,z_n)}^2 - 
2 \gamma_{n} \scal{f_n -f_K}{\nabla V(f_n, z_n)}
\end{eqnarray*}
so that taking the expectation w.r.t. $z_n$ conditioned to $S_n$  and
using the assumptions, we write
\begin{eqnarray*}
\E_{z_n}[\nor{f_{n+1} - f_K}^2|S_n]  &\leq& 
\nor{f_{n} - f_K}^2  + \gamma_n^2 \E_{z_n}[\nor{\nabla V(f_n,z_n)}^2|S_n] 
 - 2 \gamma_{n} \scal{f_n -f_K}{\E_{z_n}[\nabla V(f_n, z_n)|S_n]}\\
 &\le& \nor{f_{n} - f_K}^2  + \gamma_n^2 \frac{D\gamma_n}{(\gamma_n-M/2\gamma_n^2)}
 - 2 \gamma_{n} \scal{f_n -f_K}{\nabla I(f_n)}, 
\end{eqnarray*}
since $\E_{z_n}[\nabla V(f_n, z_n)|S_n]=\nabla I(f_n)$.
The series $\sum_n \gamma_n^2 \frac{D\gamma_n}{(\gamma_n-M/2\gamma_n^2)}$ converges
and the last inner product is positive by assumption,  so that  the Robbins-Siegmund's theorem 
implies \eqref{convsgd} and the theorem is proved. 
\end{proof}

\appendix
 
 \section{Remarks: assumptions}

\begin{itemize}

\item The assumptions will be satisfied if the loss is convex (and twice differentiable) and 
$\hh$ is compact.  
In fact,  a convex function  is always locally Lipschitz  so that 
if we restrict $\hh$ to be a compact set, $V$ satisfies \eqref{lip} for 
$$
L=\sup_{f\in \hh, z\in Z} \nor{\nabla V(f,z))}<\infty.
$$ 
Similarly since $V$ is twice differentiable  and convex, 
we have that the Hessian $H(V(f,z))$ of $V$ at any $f\in \hh$ and $z\in Z$ is identified with a bounded, 
positive semi-definite matrix,  that is  
$$\scal{f}{H(V(f,z))f}\ge 0\quad  \text\quad \nor{H(V(f,z))}\le 1<\infty,$$
for any $f\in \hh$ and $z\in Z$, where for the sake of simplicity we took the bound on the Hessian to be $1$.

\item The gradient in the SGD update rule can be replaced by a
  stochastic subgradient with little changes in the theorems.

\end{itemize}

\section{Learning Rates, Finite Sample Bounds and Complexity}

\subsection{Connections Between Different Notions of Convergence.} 
It is known that both convergence in expectation and strong convergence imply weak convergence.
On the other hand if we have weak consistency and 
$$
\sum_{n=1}^\infty \P(I(f_n)-I(f_K)>\epsilon) <\infty
$$
for all $\epsilon>0$, then weak consistency implies strong consistency by the Borel-Cantelli lemma.

\subsection{Rates and Finite Sample Bounds.} 
A stronger result is   weak convergence  with  a {\em rate}, that is 
$$
\P(I(f_n)-I(f_K)>\epsilon)\ge\delta(n,\epsilon),
$$
where $\delta(n,\epsilon)$ decreases in $n$ for all $\epsilon>0$. We make two observations.
First, one can see that the Borel-Cantelli lemma imposes a rate on the decay of $\delta(n,\epsilon)$. 
Second, typically $\delta= \delta(n,\epsilon)$ is invertible in $\epsilon$ so that 
we can write the above result as a finite sample bound 
$$
 \P(I(f_n)-I(f_K)\le \epsilon(n,\delta) )\ge 1- \delta.
$$

\subsection{Complexity and Generalization}

We say that a class of real valued functions $\cal F$ on $Z$ is uniform Glivenko-Cantelli
if the following limit exists
$$
\lim_{n\to \infty} \P\left(\sup_{F\in \cal F}|\E_n(F)-\E(F)|>\epsilon\right)=0.
$$
for all $\epsilon>0$. 
If we consider the class of functions induced by $V$ and $\hh$, that is $F(\cdot)=V(f, \cdot)$, $f\in \hh$, the above properties can be written as 
\begin{equation}\label{ugc}
\lim_{n\to \infty} \P\left(\sup_{f\in \hh}|I_n(f)-I(f)|>\epsilon\right)=0.
\end{equation}
Clearly the above property implies \eqref{empexp}, hence consistency of ERM  if $f_{\hh}$ exists and  under mild assumption on the loss
-- see previous footnote. 

It is well known that UGC classes can be completely characterized by suitable capacity/complexity measures of $\hh$. 
In particular a class of  binary valued functions is UGC if and only if the VC-dimension is finite. Similarly 
a class of  bounded functions is UGC if and only if the fat- shattering dimension is finite. See \cite{Alonetal97} and reference therein.

Finite complexity of $\hh$ is hence a sufficient condition for the consistency of $ERM$.

\subsection{Necessary Conditions}

One natural question is weather the above conditions are also necessary  for consistency of ERM 
in the sense of \eqref{wcons}, or in other words if consistency of ERM on $\hh$ implies that $\hh$ is UGC class. 

An argument in this direction is given by Vapnik which call the result the key 
theorem in learning  (together with the converse direction). Vapnik argues that \eqref{wcons} must be 
replaced by a much stronger notion of convergence essentially holding if we replace $\hh$
with $\hh_\gamma=\{f\in \hh~|~ I(f)\ge \gamma\}$, for all $\gamma$.

Another result in this direction  is given {\em without proof} in \cite{Alonetal97}.

\subsection{Robbins Siegmund's Lemma}

We use the stochastic approximation framework described by Duflo
(\cite{Duflo91}, pp 6-15).

We assume a sequence of {\it data $z_i$} defined by a
probability space ${\Omega, A, P}$ and a filtration $\F = {(\cal
  F)}_{n \in \N}$ where ${\cal F}_n$ is a $\sigma$-field and ${\cal
  F}_n \in {\cal F}_{n+1}$. In addition a sequence ${\cal X}_n$ of
measurable functions from ${\Omega, A}$ to another measurable space is
defined to be {\it adapted} to $\F$ if for all $n$, ${\cal X}_n$ is
${\cal F}_n$- measurable.
  \vspace{0.2cm}  

{\bf Definition} {\it Suppose that $X = (X_n)$ is a sequence of random
  variables adapted to the filtration $\F$. $X$ is a supermartingale
  if it is integrable (see \cite{Duflo91}) and if 
  $$
  \E[X_{n+1}| {\cal F}_n] \leq X_n$$
}

The following is a key theorem (\cite{Duflo91}).

\begin{thm}
\label{Robbins-Siegmund}
({\it Robbins-Siegmund}) Let $(\Omega, {\cal F}, P)$ be a probability
space. Let $(V_n)$, $(\beta_n)$, $(\chi_n)$, $(\eta_n)$ be finite
non-negative ${\cal F}_n$-mesurable random variables, where ${\cal
  F}_1 \subset \cdots \subset {\cal F}_n \subset \cdots$ is a sequence
of sub-$\sigma$-algebras of ${\cal F}$.  Suppose that $(V_n)$,
$(\beta_n)$, $(\chi_n)$, $(\eta_n)$ are four positive sequences
adapted to $\F$ and that

$$
\E[V_{n+1}| {\cal F}_n] \leq V_n (1 + \beta_n) + \chi_n - \eta_n.$$

Then if $\sum \beta_n < \infty$ and $\sum \chi_n < \infty$, almost
surely $(V_n)$ converges to a finite random variable and the series
$\sum \eta_n$ converges.

\end{thm}

We provide a short proof of a special case of the theorem.

\begin{thm}
\label{ Special case of Robbins-Siegmund}
{Suppose that $(V_n)$ and $(\eta_n)$ are positive sequences adapted to $\mathbb F$ and that

\[
\mathbb E[V_{n+1} ~|~ \mathcal F_n] \leq V_n - \eta_n.
\]
Then almost surely $(V_n)$ converges to a finite random variable and the series $\sum\eta_n$ converges.}

\end{thm}

\noindent\textbf{Proof}\\
\indent Let $Y_n = V_n + \sum_{k=1}^{n-1}\eta_k$. Then we have

\begin{eqnarray*}
\mathbb E[Y_{n+1} ~|~ \mathcal F_n] = \mathbb E[V_{n+1} ~|~ \mathcal F_n] + \sum_{k=1}^n\mathbb E[\eta_k~|~\mathcal F_n] \leq V_n - \eta_n + \sum_{k=1}^n\eta_k = Y_n.
\end{eqnarray*}

So $(Y_n)$ is a supermartingale, and because  $(V_n)$ and $(\eta_n)$ are positive sequences, $(Y_n)$ is also bounded from below by 0, which implies it converges almost surely. It follows that both $(V_n)$ and $\sum\eta_n$ converge.

\bibliography{WellPosedness-Consistency}

\end{document}